\documentclass{llncs}
\usepackage{makeidx}  

\usepackage{url}
\usepackage{makeidx}         
\usepackage{graphicx}        
\usepackage{multicol}        
\usepackage[bottom]{footmisc}

\usepackage{amsmath,bm,amsfonts,amssymb}

\usepackage{commath}      
\usepackage{color,xcolor,ucs}
\usepackage{subfig}
\usepackage[font=small,labelfont=bf]{caption}

\usepackage{floatrow}
\usepackage{tabularx}
\usepackage{float}
\usepackage{booktabs}
\usepackage{cite}
\usepackage{xr-hyper}
\usepackage{wrapfig}

\usepackage{algorithm}
\usepackage{algpseudocode}
\usepackage{multirow}
\usepackage{listings}
\usepackage{lipsum}
\usepackage{textcomp} 

\usepackage{array}
\DeclareMathOperator{\EX}{\mathbb{E}} 

\makeatletter
\newcommand{\clearsubcaptcounter}{\setcounter{sub\@captype}{0}}
\makeatother    

\newcommand{\B}{\normalfont\textbf}

\newcolumntype{C}[1]{>{\centering\arraybackslash}p{#1}}

\begin{document}
\frontmatter          
\pagestyle{headings}  
%
\mainmatter              
\title{Probabilistic Collision Constraint for Motion Planning in Dynamic Environments}
%
%
\author{Antony Thomas \and Fulvio Mastrogiovanni \and Marco Baglietto}
\authorrunning{Antony Thomas et al.} 
%

%
\institute{Department of Informatics, Bioengineering, Robotics, and Systems Engineering, University of Genoa, Via All'Opera Pia 13, 16145 Genoa, Italy.\\
\email{antony.thomas@dibris.unige.it, fulvio.mastrogiovanni@unige.it,  marco.baglietto@unige.it}}

\maketitle              

\begin{abstract}
Online generation of collision free trajectories is of prime importance for autonomous navigation. Dynamic environments, robot motion and sensing uncertainties adds further challenges to collision avoidance systems. This paper presents an approach for collision avoidance in dynamic environments, incorporating robot and obstacle state uncertainties. We derive a tight upper bound for collision probability between robot and obstacle and formulate it as a motion planning constraint which is solvable in real time. The proposed approach is tested in simulation considering mobile robots as well as quadrotors to demonstrate that successful collision avoidance is achieved in real time application. We also provide a comparison of our approach with several state-of-the-art methods. 
\end{abstract}

\section{Introduction}
Safe and reliable motion planning is an important problem in many robotic applications. In many real-world applications, for example, in crowded and dynamic environments such as factories or living spaces, robots are in close proximity to humans and other robots thus making online computation of collision free trajectories vital. However, robot state estimates are often uncertain due to actuation errors, imperfect sensing, and incomplete knowledge of the environment. This leads to reasoning regarding these uncertainties during motion planning, localization and while estimating the motions of other robots or humans. Therefore, the probability of colliding with obstacles is thus computed incorporating these uncertainties. Furthermore, the collision probability computation needs to be reasonably fast to be operable in real time. 

Various methods exist to model the uncertainties arising during collision avoidance planning. Similarly to other approaches~\cite{dutoit2011IEEE,patil2012ICRA,park2018IEEE,axelrod2018IJRR,zhu2019RAL,thomas2020IRIM}, in this paper we model the uncertainties using Gaussian distributions. The robot and obstacle locations are thus parameterized as Gaussian probability distribution functions (pdfs). Moreover, uncertain environments are such that they often preclude the existence of collision free trajectories~\cite{aoude2013AR}. For example, the Gaussian pdf representation assigns to both robot and obstacle locations a non-zero probability anywhere in the environment. Thus, suitable trajectories are often computed such that the collision probability falls within some threshold. Yet, most approaches tend to over estimate the collision probability and thereby compute sub-optimal trajectories or in some cases pronounce plans to be infeasible. Thus, to compute safe and efficient trajectories, computing tight collision probability bounds is essential.

In this paper, we develop an accurate constraint for collision avoidance during motion planning. For a specific collision probability threshold, the collision avoidance constraint can then be used for online Model Predictive Control (MPC) optimization. To be robust to uncertain environments, robot motion and sensing uncertainties (and obstacle uncertainties) are incorporated by propagating the uncertainties within the MPC framework. The resulting collision probability bound is more tighter as compared to prior approaches. The proposed approach is valid in both 2D and 3D domains subject to the assumption of spherical geometry for robot, obstacles and we evaluate our method in Gazebo based simulations (see Fig.~\ref{fig:exp}) in multi-robot setting with both mobile robots and quadrotors. 
\begin{figure}[t!]
  \subfloat[]{\includegraphics[scale=0.14]{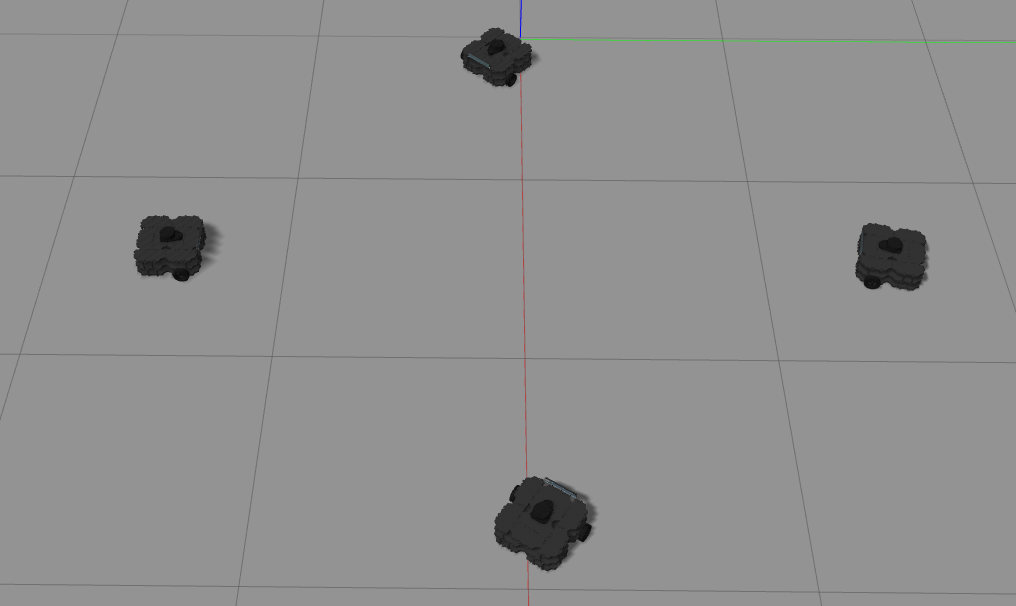}\label{fig:turtlebot}}\hspace{0.2cm}
  \subfloat[]{\includegraphics[scale=0.171]{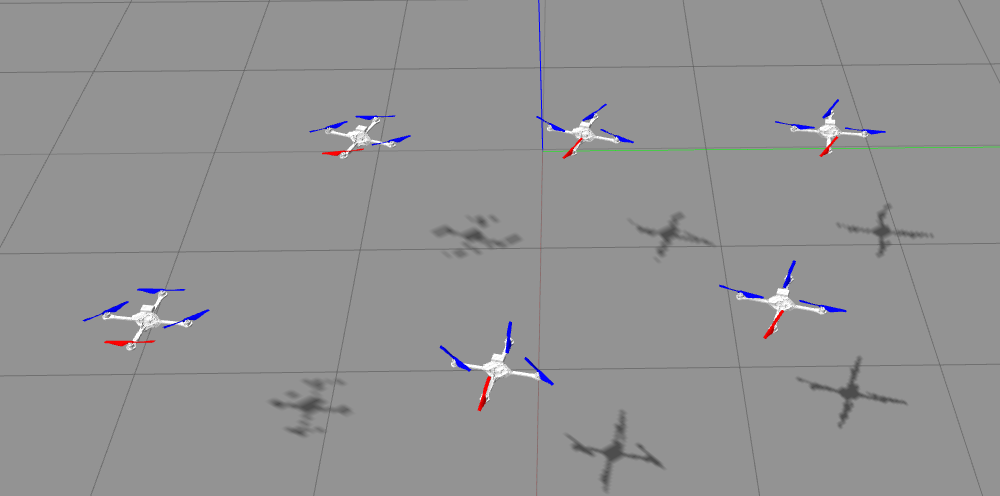}\label{fig:hummingbird}}
   \caption{Snapshots from experiments with a team of (a) mobile robots and (b) quadrotors.}
  \label{fig:exp}
\end{figure}

The paper is organized as follows. In Section~\ref{sec:related} we discuss some of the prior approaches for probabilistic collision avoidance. Notations are introduced in Section~\ref{sec:prelim}. We derive our probabilistic collision constraint and provide a comparison with other prior methods in Section~\ref{sec:approach}. In Section~\ref{sec:results} we validate our method under different multi-robot scenarios. 

\section{Related Work}
\label{sec:related}
In this section we provide a brief overview of existing works on collision probability computation for motion planning. Several approaches exist that incorporate various uncertainties to compute collision probabilities. The starting point of all approaches is to formulate the collision constraint, that is, the conditions under which a collision occurs between a robot and an obstacle. The methods differ in their formulation of collision constraint due to different assumptions regarding the shape of the robot and the obstacles (for example, point, spherical, ellipsoidal, or rectangular shapes for the robot and the obstacle), modeling of uncertainties. Overall, most approaches tend to be overly conservative and provide loose upper bounds. This can lead to sub-optimal plans or in some case render plans infeasible and thus it is desirable to compute accurate collision probabilities to ensure safe and efficient trajectories. 

Bounding volume approaches enlarge the robot and obstacles with a larger volume that them, mostly employing their 3-$\sigma$ uncertainty ellipsoids~\cite{bry2011ICRA,kamel2017IROS}. In~\cite{park2012ICAPS} exact collision checking with conservative bounds on the position of the moving obstacles is performed. Convex hulls of robot links enlarged according to the sigma standard deviation are used in~\cite{lee2013IROS}. An upper bound for collision probability is computed using rectangular bounding boxes for both the robot and the obstacle in~\cite{hardy2013TRO}. Patil \textit{et al.}~\cite{patil2012ICRA} truncate the \textit{a priori} Gaussian robot state distributions~\cite{johnson1994truncatedGaussian} that collide with the obstacle. The truncated distributions account for collision free samples and are then propagated to obtain collision probability estimates. Truncating propagated distributions is also employed in~\cite{liu2014ICRA} to compute risk-aware and asymptotically optimal trajectories.  

Exact collision probability can be computed by marginalizing the joint distribution between the robot and obstacle locations~\cite{dutoit2010ICRA}. This integration is then performed over the set of robot and obstacle locations that satisfy the collision constraint. However, there is no closed form solution to the integral and numerical integration or Monte Carlo (MC) techniques are employed~\cite{schmerling2017RSS}. Yet, numerical integration tend to be computationally expensive and therefore approximate MC methods are employed. MC integration results in a double summation and this is approximated to a single summation in~\cite{lambert2008ICCARV}. Another related approach is the Monte Carlo Motion Planning (MCMP) approach of Janson \textit{et al.}~\cite{janson2018ISRR}. They solve a deterministic motion planning problem with inflated obstacles. Most often, it is convenient to model the collision condition as a constraint of an optimal control problem. Yet, MC approaches tend to be computationally expensive and hard to model within an optimization framework. Assuming that the robot size is negligible, the integration can be approximated to a product of the joint distribution evaluated at the robot and obstacle mean locations and the volume occupied by the robot. Du Toit and Burdick~\cite{dutoit2011IEEE} compute the joint density evaluated at the robot center, whereas in~\cite{park2018IEEE} an upper bound is computed by evaluating the density at the surface of the robot. 

Assuming the robot and the obstacles to be spherical objects, the distance between spheres is used to formulate the collision constraint as a quadratic form in random variables in~\cite{thomas2020IRIM,thomas2021ISR}. The cumulative distribution function (cdf) of the quadratic form gives the required collision probability. Another popular approach is based on chance-constraints. These methods find an optimal sequence of control inputs subject to constraints, that is, collision probability thresholds~\cite{blackmore2011TRO}. Convex polygonal obstacles are considered in~\cite{blackmore2011TRO} and a Gaussian parameterization allows for obtaining constraints in terms of the mean and covariance of the robot and obstacle locations. Chance-constraints are employed in~\cite{zhu2019RAL} by linearizing the nonlinear collision constraints to derive an approximate upper bound. Newton's method combined with chance-constraints is used in~\cite{sun2016ISRR} to obtain an upper bound for collision probability. Future obstacle trajectories are predicted using a Gaussian Process (GP) based technique that learns the mapping from states to trajectory derivatives in~\cite{aoude2013AR}. In~\cite{frey2020RSS} the first-exist times for Brownian motions are leveraged to compute collision probabilities.~\cite{axelrod2018IJRR} focus exclusively on environment uncertainty and formalize a notion of \textit{shadows}, a geometric equivalent of confidence intervals for uncertain obstacles. Though shadows fundamentally give rise to loose bounds, the computational complexity of bounding the collision probability is greatly reduced. To incorporate uncertainties, obstacles are modelled as polytopes with Gaussian-distributed faces in~\cite{shimanuki2018WAFR}. 

Risk-aware motion planning assigns risk to certain regions and paths are considered by the planning algorithm only if the risk obtained in different regions are below some predefined thresholds~\cite{chow2017JMLR}. Planning a collision-free path of a robot in the presence of \textit{risk zones} by penalizing the time spent in these zones is presented in~\cite{salzman2017ICAPS}. Jasour \textit{et al.}\cite{jasour2019RSS} employ risk contours map that takes into account the risk information (uncertainties in location, size and geometry of obstacles) to obtain safe paths with bounded risks. A related approach for randomly moving obstacles is investigated in~\cite{hakobyan2019RAL}. Collision avoidance in the context of human-robot interaction are presented in~\cite{bajcsy2019ICRA,fridovich2020IJRR}. Formal verification methods have also been used to construct safe plans~\cite{ding2013ICRA, sadigh2016RSS}. 

\section{Preliminaries}
\label{sec:prelim}
Throughout this paper vectors will be assumed to be column vectors and will be denoted by bold lower case letters, that is, $\textbf{x}$ and its components will be denoted by lower case letters. Transpose of $\textbf{x}$ will be denoted by $\textbf{x}^T$ and its Euclidean norm by $\norm{\textbf{x}} = \sqrt{\textbf{x}^T\textbf{x}}$. The mean of a random vector will be denoted by $\bm{\mu}$, the corresponding covariance by $\Sigma$ and the expected value by $\EX(\B{x})$. A multivariate Gaussian distribution of $\B{x}$ with mean $\bm{\mu}$ and covariance $\Sigma$ will be denoted using the notation $\B{x} \sim \mathcal{N}(\bm{\mu}, \Sigma)$. Matrices will be denoted by capital letters, that is, $M$, with its trace denoted by $tr(M)$. The identity matrix will be denoted by $I$ or $I_n$ when the dimension needs to be stressed. A diagonal matrix with diagonal elements $\lambda_1, \ldots, \lambda_n$ will be denoted by $diag(\lambda_1, \ldots, \lambda_n)$. Sets will be denoted using mathcal fonts, that is, $\mathcal{S}$. Unless otherwise mentioned, subscripts on vectors/matrices will be used to denote time indexes and (whenever necessary) superscripts will be used to indicate the robot or the object that it represents. For example, $\textbf{x}_k^i$ represents the state of robot $i$ at time $k$. The notation $P(\cdot)$ will be used to denote the probability of an event and the pdf will be denoted by $p(\cdot)$.

At any time $k$, we denote the robot state by $\textbf{x}_k$, the acquired measurement from objects is denoted by $\textbf{z}_k$ and the applied control action is denoted as $\textbf{u}_k$. We also make the following assumptions: (1) the uncertainties are modeled using Gaussian distributions, (2) the robot and obstacles are assumed to be non-deformable spherical objects.  

To describe the dynamics of the robot, we consider a standard motion model with Gaussian distributed noise
\begin{equation}
\textbf{x}_{k+1} = f(\textbf{x}_k,\textbf{u}_k) + n_{k}\  ,  \ n_{k} \sim \mathcal{N}(0,R_{k})
\label{eq:odometry_model}
\end{equation}
\noindent where $n_k$ is the random unobservable noise, modeled as a zero mean Gaussian distribution. Objects are detected through the robot's sensors and assuming known data association, the observation model can be written as  
\begin{equation}
\B{z}_k = h(\B{x}_k) + v_k \  ,  \ v_k \sim \mathcal{N}(0,Q_k)
\label{eq:measurement_model}
\end{equation}

\section{Probabilistic Collision Avoidance}
\label{sec:approach}
\subsection{Collision Constraint}
We denote by $\mathcal{R}$ the set of all points occupied by a rigid-body robot at any given time. Similarly, let $\mathcal{S}$ represent the set of all points occupied by a rigid-body obstacle. A collision occurs if there exits a point such that it is in both $\mathcal{R}$ and $\mathcal{S}$. Thus the collision condition is defined as
\begin{equation}
\mathcal{R} \cap \mathcal{S} \neq \{\phi\}
\end{equation}
\noindent and we denote the probability of collision as $P\left(\mathcal{R} \cap \mathcal{S} \neq \{\phi\}\right)$. In this work we assume spherical geometries for $\mathcal{R}$ and $\mathcal{S}$ with radii $r_1$ and $s_1$, respectively. We assign body-fixed reference frames to robot and obstacle centers located at $\B{x}_k$ and $\B{s}_k$, respectively in the global frame. By abuse of notation we will use $\B{x}_k$ and $\B{s}_k$ equivalently to $\mathcal{R}$ and $\mathcal{S}$. However, when we talk about the distribution of their locations, we refer to the distribution of their centers (the body-fixed frame). The collision condition is thus defined in terms of the body-fixed frames as
\begin{equation}
\mathcal{C}_{\B{x}_k,\B{s}_k}: \mathcal{R} \cap \mathcal{S} \neq \{\phi\}
\end{equation}
We recall here that the locations of the robot and the obstacles are in general uncertain. Let us now consider a robot and an obstacle at any given time instant $k$, distributed according to the Gaussians $\B{x}_k \sim \mathcal{N}\left(\bm{\mu}_{\B{x}_k},\Sigma_{\B{x}_k} \right)$ and $\B{s}_k \sim \mathcal{N}\left(\bm{\mu}_{\B{s}_k},\Sigma_{\B{s}_k}\right)$, respectively. Since the robot and the obstacles are assumed to be spherical objects, the collision constraint is written as
\begin{equation}
\norm{\B{x}_k -\B{s}_k}^2 \leq (r_1+s_1)^2
\label{eq:coll_condition}
\end{equation} 
\noindent Thus~\ref{eq:coll_condition} is equivalent to $\mathcal{C}_{\B{x}_k,\B{s}_k}$. Let us denote the difference between the two random variables by $\B{w} = \B{x}_k -\B{s}_k$. Using the expression for the difference between two Gaussian distributions, we have $\B{w}  \sim \mathcal{N} \left(\bm{\mu}_{\B{x}_k} - \bm{\mu}_{\B{s}_k}, \Sigma_{\B{x}_k} + \Sigma_{\B{s}_k} \right)$. The collision constraint in~(\ref{eq:coll_condition}) can now be written in terms of $\B{w}$,
\begin{equation}
\B{y} = \norm{\B{w}}^2 = \B{w}^T\B{w} \leq (r_1+s_1)^2
\label{eq:collision}
\end{equation}
\noindent where $\B{y}$ is a random vector distributed according to the squared $L_2$-norm of $\B{w}$. 

\begin{proposition}
A symmetric matrix $A \in \mathbb{R}^{n\times n}$ with orthonormal eigenvectors $q_i$ can be factorized as
\begin{equation}
A = Q\Lambda Q^T
\end{equation}
\noindent where the columns of $Q$ correspond to the orthonormal eigenvectors $q_i$ and $\Lambda$ is a diagonal matrix comprised of the corresponding eigenvalues of $A$.
\label{prop:eigen_decom}
\end{proposition}
\begin{lemma}
For a symmetric matrix $A \in \mathbb{R}^{n\times n}$ and a random vector $\B{x}$, we have
\begin{equation}
\B{x}^TA\B{x} \leq \lambda_{max}\norm{\B{x}}^2
\end{equation}
\noindent where $\lambda_{max}$ is the maximum eigenvalue of $A$.
\label{lm:inequality}
\end{lemma}
\begin{proof}
From Proposition~\ref{prop:eigen_decom}, we have
\begin{equation*}
\begin{split}
\B{x}^TA\B{x} &= \B{x}^TQ\Lambda Q^T\B{x}
= \left(Q^T\B{x}\right)^T\Lambda \left(Q^T\B{x}\right)
= \sum_{i=1}^n \lambda_i(q_i^T\B{x})^2\\
&\leq \lambda_{max} \sum_{i=1}^n (q_i^T\B{x})^2
= \lambda_{max} \norm{\B{x}}^2\\
\end{split}
\end{equation*}
\noindent where we have used the fact that the eigenvectors are orthonormal.
\end{proof}
\begin{proposition}
For any random variable $\B{x}$, the probability of an event $P(\B{x}\leq x)$ is given by its cumulative distribution function (cdf) $F_{\B{x}}(x)$, that is,
\begin{equation}
F_{\B{x}}(x) = P(\B{x}\leq x), \quad - \infty < x < +\infty
\end{equation}
\end{proposition}
\begin{proposition}
For n-dimensional $\B{x} \sim \mathcal{N}(\bm{\mu}, \Sigma)$
\begin{equation}
y = (\B{x}- \bm{\mu})^T\Sigma^{-1}(\B{x}- \bm{\mu}) \sim \chi^2_n
\end{equation}
\noindent where $\chi^2_n$ denotes the chi-squared distribution with
$n$ degrees of freedom.
\end{proposition}

Let $F_{chi}$ be the cdf of a chi-squared distribution with $n$ degrees of freedom and let $\B{x} \sim \mathcal{N}(\bm{\mu}, \Sigma)$, then for any $- \infty < x < +\infty$ we have
\begin{equation}
P((\B{x}- \bm{\mu})^T\Sigma^{-1}(\B{x}- \bm{\mu}) \leq x) = F_{chi}(x)
\end{equation} 
Alternatively, for any $0 \leq \epsilon \leq 1$, we have
\begin{equation}
P((\B{x}- \bm{\mu})^T\Sigma^{-1}(\B{x}- \bm{\mu}) \leq F_{chi}^{-1}(\epsilon)) = \epsilon
\end{equation}
\begin{lemma}
Let $\B{x} \sim \mathcal{N}(\bm{\mu}, \Sigma)$, $\norm{\B{x}}^2 \leq \alpha $, $F_{chi}$ be the cdf of a chi-squared distribution and
\begin{equation*}
P((\B{x}- \bm{\mu})^T\Sigma^{-1}(\B{x}- \bm{\mu}) \leq F_{chi}^{-1}(\epsilon)) = \epsilon
\end{equation*}
\noindent Then $\lambda_{max}\left(\alpha - 2\B{x}^T\bm{\mu} + \bm{\mu}^T\bm{\mu}\right) \leq F_{chi}^{-1}(\epsilon)$, where $\lambda_{max}$ is the maximum eigenvalue of $\Sigma^{-1}$.
\label{lm:prob_limit}
\end{lemma}
\begin{proof}
From Lemma~\ref{lm:inequality}, it follows that
\begin{equation}
(\B{x}- \bm{\mu})^T\Sigma^{-1}(\B{x}- \bm{\mu}) \leq \lambda_{max}\norm{\B{x}- \bm{\mu}}^2
\label{eq:norm_max}
\end{equation}
\noindent where $\lambda_{max}$ is the maximum eigenvalue of $\Sigma^{-1}$. Expanding the right-hand side of~(\ref{eq:norm_max}) and using the fact that $\norm{\B{x}}^2 \leq \alpha $, we get
\begin{equation}
(\B{x}- \bm{\mu})^T\Sigma^{-1}(\B{x}- \bm{\mu}) \leq \lambda_{max}\left(\alpha - 2\B{x}^T\bm{\mu} + \bm{\mu}^T\bm{\mu}\right)
\end{equation}
\noindent Thus, for $(\B{x}- \bm{\mu})^T\Sigma^{-1}(\B{x}- \bm{\mu}) \leq F_{chi}^{-1}(\epsilon)$ it suffices that $\lambda_{max}\left(\alpha - 2\B{x}^T\bm{\mu} + \bm{\mu}^T\bm{\mu}\right) \leq F_{chi}^{-1}(\epsilon)$.
\end{proof}
Note that the collision constraint in~(\ref{eq:collision}) is of the form $\norm{\B{x}}^2 \leq \alpha $ and this allows us to define a notion of maximum allowable collision probability.
We remind the readers that $\mathcal{C}_{\B{x}_k,\B{s}_k}$ represents the collision condition and is therefore equivalent to~(\ref{eq:coll_condition}). We now define the collision constraint that satisfy the required collision probability threshold.
\begin{lemma}
Given n-dimensional $\B{x} \sim \mathcal{N}(\bm{\mu}, \Sigma)$, and $P(\ \norm{\B{x}}^2 \leq \alpha) \leq \epsilon$, then
\begin{equation}
\lambda_{max}\left(\alpha - 2\B{x}^T\bm{\mu} + \bm{\mu}^T\bm{\mu}\right) \leq F_{chi}^{-1}(\epsilon)
\label{eq:chance_constraint}
\end{equation}
\noindent where $\lambda_{max}$ is the maximum eigenvalue of $\Sigma^{-1}$ and $F_{chi}$ is the cdf of the chi-squared distribution with n degrees of freedom.
\label{lm:constraint}
\end{lemma}

\begin{proof}
We have
\begin{equation*}
\begin{split}
P(\ \norm{\B{x}}^2 \leq \alpha) &= P(\B{x}^T\B{x} \leq \alpha) = P((\B{x} - \bm{\mu} +\bm{\mu})^T(\B{x} - \bm{\mu} +\bm{\mu}) \leq \alpha)\\
& = P((\B{x} - \bm{\mu})^T(\B{x} - \bm{\mu}) + 2(\B{x}-\bm{\mu})^T\bm{\mu} + \bm{\mu}^T\bm{\mu} \leq \alpha)\\
&= P(\ \norm{\B{x}- \bm{\mu}}^2  \leq \alpha - 2\B{x}^T\bm{\mu} + \bm{\mu}^T\bm{\mu})\\
&= P(\lambda_{max}\norm{\B{x}- \bm{\mu}}^2  \leq \lambda_{max}(\alpha - 2\B{x}^T\bm{\mu} + \bm{\mu}^T\bm{\mu}))
\end{split}
\end{equation*}
\noindent where $\lambda_{max}$ is the maximum eigenvalue of $\Sigma^{-1}$. Now from Lemma~\ref{lm:inequality} it follows that
\begin{multline*}
P(\lambda_{max}\norm{\B{x}- \bm{\mu}}^2  \leq \lambda_{max}(\alpha - 2\B{x}^T\bm{\mu} + \bm{\mu}^T\bm{\mu})) \\= P((\B{x}- \bm{\mu})^T\Sigma^{-1}(\B{x}- \bm{\mu}) \leq \lambda_{max}(\alpha - 2\B{x}^T\bm{\mu} + \bm{\mu}^T\bm{\mu}))
 = P(\norm{\B{x}}^2 \leq \alpha) = \epsilon
\end{multline*}
\noindent The required result then directly follows from Lemma~\ref{lm:prob_limit}.
\end{proof}
Since the collision constraint in~(\ref{eq:collision}) is a Gaussian distribution, for a collision probability threshold of $\epsilon$ we can directly use the constraint in~(\ref{eq:chance_constraint}). We now state the following lemma which is a direct consequence of Lemma~\ref{lm:constraint}.
\begin{lemma}
Given n-dimensional $\B{x} \sim \mathcal{N}(\bm{\mu}, \Sigma)$, and $P(\ \norm{\B{x}}^2 \leq \alpha) \leq \epsilon$, then
\begin{equation}
P(\ \norm{\B{x}}^2 \leq \alpha)  \leq F_{chi}^{-1}\left(\lambda_{max}(\alpha - 2\B{x}^T\bm{\mu} + \bm{\mu}^T\bm{\mu})\right)
\label{eq:corollary}
\end{equation}
\end{lemma}

\subsection{Objective Function}
We formulate the collision avoidance problem as an optimization problem. At each time instant $k$, the robot plans for $L$ look-ahead steps and minimizes an objective function $J_k$, subject to collision and other constraints. The optimization problem can then be formally stated as 
\begin{equation}
\begin{split}
& \underset{\B{x}_{k:k+L}, \B{u}_{k:k+L-1}}{ \min} \quad J_k\\
& s.t. \quad \textbf{x}_{k+1} = f(\textbf{x}_k,\textbf{u}_k)\\
& \quad \ \, \quad \B{u}_{k+l} \in \B{U} \\
& \quad \ \, \quad P\left(\mathcal{C}_{\B{x}_{k+l},\B{s}_k^i}\right) \leq \epsilon
\end{split}
\label{eq:cost_fn}
\end{equation}
\noindent where 
\begin{equation}
J_k = \sum_{l=0}^{L-1} c_l(\B{x}_{k+l},\B{u}_{k+l}) + c_L(\B{x}_{k+L})
\end{equation}
\noindent with $c_l$ denoting the cost term at time $k+l$ and $c_L$ denoting the terminal cost, $\textbf{x}_{k+1} = f(\textbf{x}_k,\textbf{u}_k)$ is the robot dynamics~(\ref{eq:odometry_model}), $\B{u}_{k+l} \in \B{U}$ constraints the control inputs to lie within the feasible set $\B{U}$ and $P\left(\mathcal{C}_{\B{x}_k,\B{s}_k^i}\right) \leq \epsilon$ enforces a collision probability threshold of $\epsilon$ with obstacles $\B{s}_k^i$. 

We recall here that to determine the constraint $P\left(\mathcal{C}_{\B{x}_{k+l},\B{s}_k^i}\right) \leq \epsilon$ in~(\ref{eq:cost_fn}) it is required to evaluate the constraint in~(\ref{eq:chance_constraint}), which depends on the uncertainty or the covariance at each time step. Thus the uncertainty needs to be propagated at each time step to compute the collision probability constraint. In this paper we use the Extended Kalman Filter (EKF) uncertainty propagation; other approaches can be found in~\cite{luo2017PAS}. Note that the covariance dynamics dependent on the robot state and control inputs and hence require $\frac{L}{2}(n_{\B{x}}^2 + n_{\B{x}})$~\cite{hewing2018ECC} ($n_{\B{x}}$ is the dimension of $\B{x}$) additional variables in the optimization problem, increasing the computation time significantly. Thus, similar to~\cite{hewing2018ECC, zhu2019RAL}, we approximate the uncertainty evolution by propagating the robot uncertainties based on its last-loop state and control inputs.

\subsection{Comparison to Other Approaches}
We provide a comparison with several state-of-the-art methods using a robot and a close-by obstacle. For this comparison, we use a 2D example, however our approach is not limited to 2D scenarios and is equally applicable in 3D scenarios as it can be seen in Section~\ref{sec:results}. The robot is located at $(0.38, 0)$ m with radius $0.2$ m and covariance $diag(0.04, 0.04)$ $\textrm{m}^2$. The obstacle is located at the origin with radius $0.2$ m. The collision probability values can be seen in Table~\ref{table1}. To  validate the value computed using our approach, we compute the exact collision probability by performing numerical integration. Given the current robot state $\B{x}_k$ and the obstacle state $\B{s}_k$, the collision probability is given by
\begin{equation}
P\left(\mathcal{C}_{\B{x}_k,\B{s}_k}\right) = \int_{\B{x}_k} \int_{\B{s}_k} I_c(\B{x}_k,\B{s}_k)p(\B{x}_k,\B{s}_k)
\label{eq:numerical}
\end{equation}
where $I_c$ is an indicator function defined as
\begin{equation}
   I_c(\B{x}_k,\B{s}_k)= 
   \begin{cases}
     1 \ &\text{if} \ \mathcal{R} \cap \mathcal{S} \neq \{\phi\} \\
     0 \ &\text{otherwise}.
   \end{cases}
\end{equation}
\noindent and $p(\B{x}_k,\B{s}_k)$ is the joint distribution of the robot and the obstacle. The numeric integral of~(\ref{eq:numerical}) gives the exact value and is used to compare the tightness of the upper bound computed using our approach. As seen from Table~\ref{table1} the value computed using our approach provides a tighter bound when compared to other approaches. The double summation of numerical integration is approximated to a single summation in~\cite{lambert2008ICCARV} and this results in a much higher value. Other approaches compute loose upper bounds and hence the resulting values are significantly higher. Our approach thus computes a tighter upper bound.

\begin{table}
\small\sf\centering
 \caption{Comparison of collision probability methods.}
\scalebox{1}{
\begin{tabular}{ |c|c|c| } 
 \hline
 Methods & Collision  & Computation \\
 & probability & time (ms)  \\
 \hline 
 Numerical integral & 0.1728 & 9168.9 $\pm$ 258.0 \\ 
  \hline 
   Approximate Numerical integral~\cite{lambert2008ICCARV} & 0.4280 & 18.30 $\pm$ 3.90 \\ 
  \hline 
 Bounding volume~\cite{park2012ICAPS,kamel2017IROS} & 1 & 0.1480 $\pm$ 0.4411\\ 
 \hline
 Maximum probability approximation~\cite{park2018IEEE}  & 1  & 101.6 $\pm$ 23.86 \\
  \hline
 Chance constraint~\cite{zhu2019RAL} & 0.5398 & 0.3917 $\pm$ 0.1278 \\
  \hline
 Rectangular bounding box~\cite{hardy2013TRO} & 0.1601 & 0.067 $\pm$ 0.0070 \\
  \hline
 Our approach & 0.1772 & 0.588 $\pm$ 0.13 \\
  \hline
\end{tabular}}
 \label{table1}
\end{table}

\section{Results}
In this section we describe our implementation and then evaluate the capabilities of our approach. Simulations are performed in the Gazebo environment with mobile robots as well as quadrotors. The mobile robot kinematics is as follows
\begin{equation}
\B{x}_{k+1}  = \begin{bmatrix}
 x_k - \frac{v_k}{\omega_k}\sin(\theta_k) + \frac{v_k}{\omega_k}\sin(\theta_k + \omega_k \Delta t)\\
 y_k + \frac{v_k}{\omega_k}\cos(\theta_k) - \frac{v_k}{\omega_k}\cos(\theta_k + \omega_k \Delta t)\\
 \theta_k +  \omega_k \Delta t
\end{bmatrix} + n_k
\label{eq:kinematics}
\end{equation} 
\noindent where the applied control $\B{u}_k = (v_k, \omega_k)^T$ is made up of the linear and angular velocities and $n_k$ is the noise as defined in Section~\ref{sec:prelim}. We refer the readers to~\cite{falanga2018IROS} for the quadrotor dynamics. The ground truth odometry from Gazebo is used to measure the pose of the robot, mimicking a motion capture system. This measurement is then corrupted with noise which is zero mean and is used to estimate the state of the robots employing an EKF. The optimization of~(\ref{eq:cost_fn}) is set up in ACADO~\cite{Houska2011OCAM} which generates a C++ template to run the MPC problem~(\ref{eq:cost_fn}), and is then modified according to the execution platform. For quadrotor control, we use the publicly available RPG-MPC\footnote{\url{https://github.com/uzh-rpg/rpg_mpc}}~\cite{falanga2018IROS} which is also based on ACADO and modify it to meet our requirements. A look-ahead horizon of $L=1$ second is used with a discretization of 0.1 seconds for mobile robots and for the quadrotors we use $L=2$ seconds. The performance is evaluated on an Intel{\small\textregistered} Core i7-6500U CPU$@$2.50GHz$\times$4 with 8GB RAM under Ubuntu 16.04 LTS.

\noindent \textbf{Comparison to bounding volume approaches:}
Bounding volume methods~\cite{park2012ICAPS,kamel2017IROS} presents a straightforward approach for computing collision probability under uncertainty by enlarging the robot and obstacles by their 3$-\sigma$ ellipsoids. We provide an efficiency comparison of such methods with ours. We consider a scenario in which a mobile robot navigates from $(0, 0)$ m to $(3, 0)$ m with an obstacle of radius $0.2$ m at $(1.5, 0)$ m. We begin with a measurement noise of $\Sigma = diag(0.02 \textrm{m}^2, 0.02 \textrm{m}^2, 1.2\textrm{deg}^2)$ and increase it to 4$\Sigma$ and 16$\Sigma$. We define the following metrics to compare the efficiency: $d-$ minimum distance between the robot and the obstacle, $l-$ total trajectory length, and $T-$ total trajectory duration. The results are shown in Table~\ref{table:result1} where each given value is an average over 10 different simulations. In all the three cases, the trajectory length and duration quantities certify our approach as most efficient. This is more evident as the measurement noise increases as we compute a tight upper bound. 
\begin{table}[t]
\small\sf\centering
 \caption{Collision probability efficiency with varying measurement noise.}
\begin{tabular}{|c|C{0.80cm}C{0.80cm}C{0.80cm}|C{0.75cm}C{0.75cm}C{0.75cm}|C{0.75cm}C{0.75cm}C{0.75cm}|  }
\hline
 {} & \multicolumn{9}{|c|}{Measurement noise} \\\cline{2-10}
 {}  & \multicolumn{3}{|c|}{$\Sigma$} & \multicolumn{3}{|c|}{4$\Sigma$} & \multicolumn{3}{|c|}{16$\Sigma$} \\
 \hline
 Method & $d$(m) & $l$(m) & $T$(s) & $d$(m) & $l$(m) & $T$(s) & $d$(m) & $l$(m) & $T$(s) \\
 \hline
 Bounding volume~\cite{park2012ICAPS,kamel2017IROS} &0.50 & 3.31 &16.19 & 0.70 & 3.90 & 16.34 & 0.81& 4.26 & 16.54\\
 \hline
  Our method & 0.49 & 3.28 & 16.01 & 0.58 & 3.59 & 16.13 & 0.61 & 3.71 & 16.29\\
  \hline
 \end{tabular}
 \label{table:result1}
\end{table}

\noindent \textbf{Mobile robot scenarios:} 
In this setting we consider multiple mobile robots exchanging their initial positions with each robot considering every other robot as a dynamic obstacle. To this end, the trajectory (pose and covariance) of each robot is communicated to other robots. The trajectories for two and four mobile robots exchanging their positions can be seen in Fig.~\ref{fig:mob2} and Fig.~\ref{fig:mob4}. We use a collision probability threshold of $0.1$ and a minimum separation of $0.2$ meters is achieved between the robots. In both the cases, a measurement noise of $\Sigma = diag(0.02 \textrm{m}^2, 0.02 \textrm{m}^2, 1.2\textrm{deg}^2)$ is used to corrupt the ground truth odometry which is then used to estimate the robot states using EKF. The simulation was run 10 times and the robots successfully avoided collisions in all the runs. For collision probability thresholds above $0.2$ the success rate was less $100\%$. The average computation time for MPC planning is $9.50$ ms and Fig.~\ref{fig:mob} shows the mean MPC planning time for each robot. The low computation time thus allows for real time online planning. 
\begin{figure}[t!]
  \subfloat[]{\includegraphics[scale=0.25]{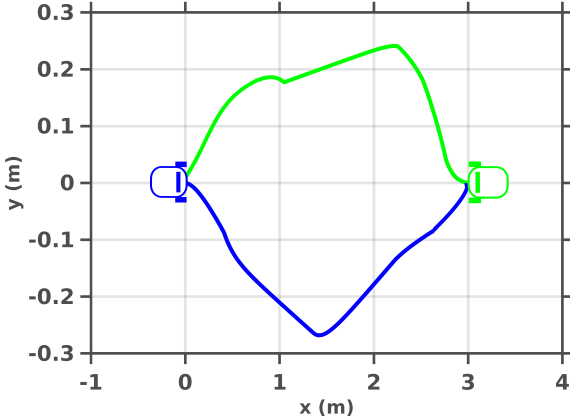}\label{fig:mob2}}\hfill
  \subfloat[]{\includegraphics[scale=0.25]{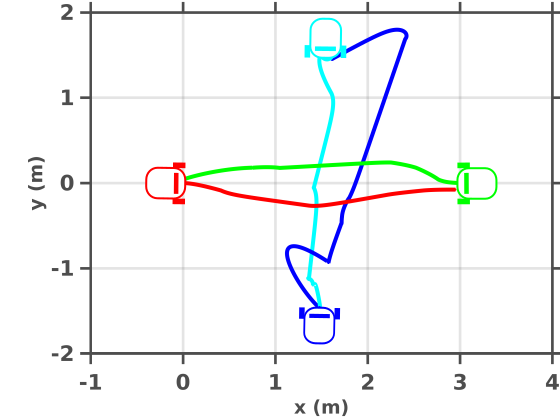}\label{fig:mob4}}\hfill
  \subfloat[]{\includegraphics[scale=0.25]{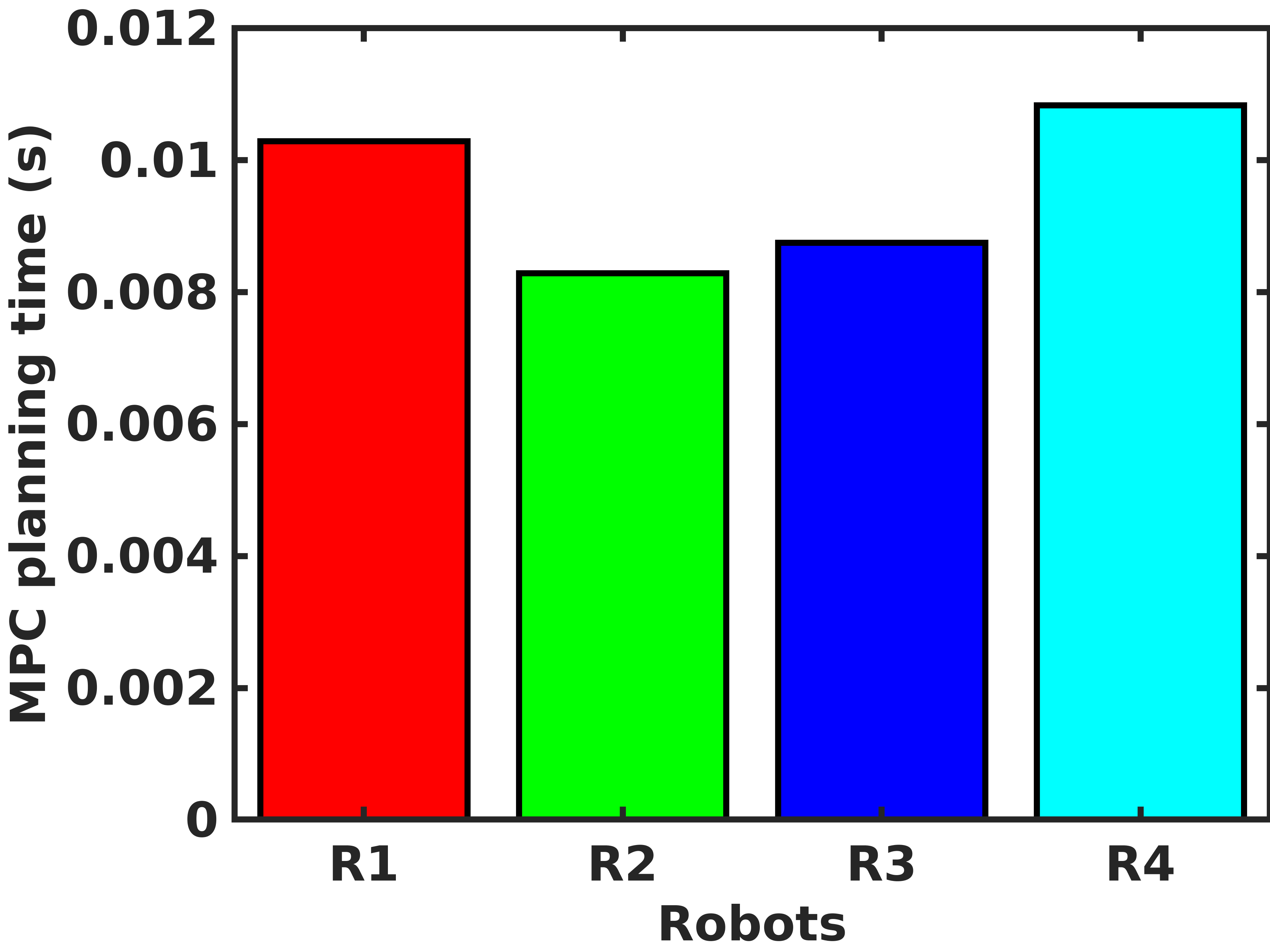}\label{fig:mob}}
   \caption{(a), (b) Simulation results of mobile robots exchanging their positions. The solid lines represent the trajectories executed by the robots. (c) Mean MPC planning time for four robots R1, R2, R3 and R4, respectively.}
  \label{fig:mobile}
\end{figure}

\noindent \textbf{Quadrotor scenarios:}
Similar to the scenario discussed above, here we consider multiple quadrotors exchanging their initial positions. Each quadrotor communicates its trajectory, both pose and covariance, with others. The top view and side view for four and six quadrotors exchanging their initial positions can be seen in Fig.~\ref{fig:qp}. Close distances between quadrotors are observed due to our tight bound. We also observed that a success rate of $100\%$ is achieved for collision probability thresholds below $0.1$. Mean computation time for MPC planning is $3.05$ ms.

\begin{figure}[h]
  \subfloat{\includegraphics[scale=0.35]{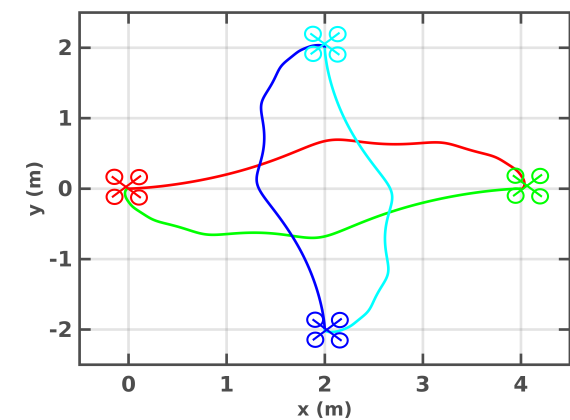}\label{fig:q1}}\hspace{0.2cm}
  \subfloat{\includegraphics[scale=0.35]{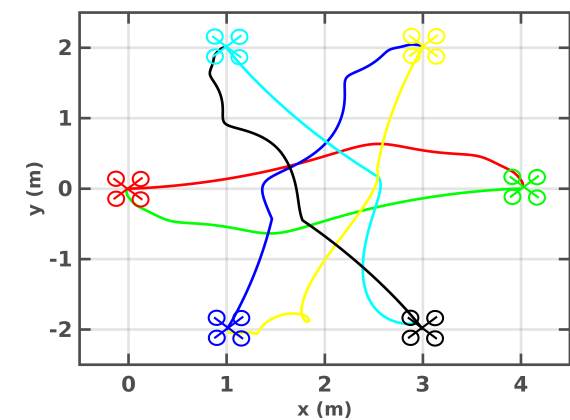}\label{fig:q2}}\\
  \vspace{-0.3cm}
  \clearsubcaptcounter
  \subfloat[Four quadrotors]{\includegraphics[scale=0.33]{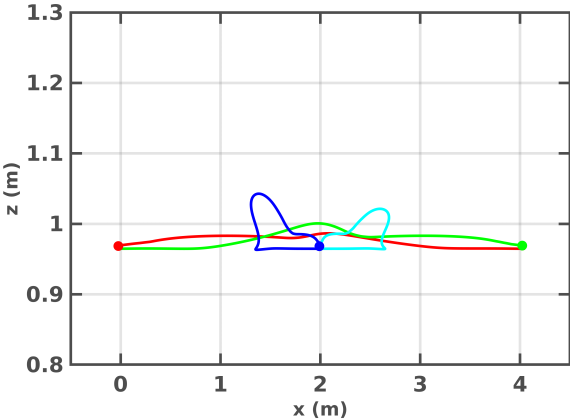}\label{fig:q3}}\hspace{0.4cm}
  \subfloat[Six quadrotos]{\includegraphics[scale=0.33]{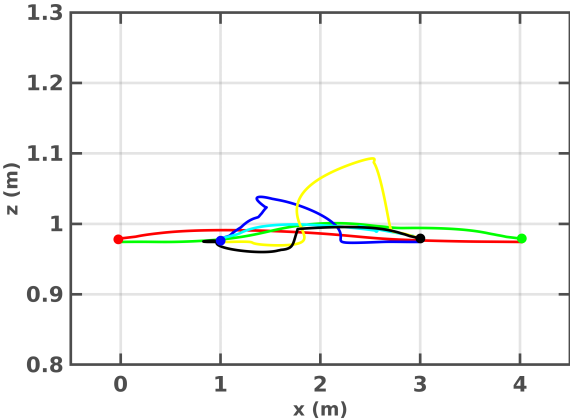}\label{fig:q4}}
   \caption{Simulation results of quadrotors exchanging their positions; trajectories executed visualized by solid lines. The upper plots show the top view (x-y) and the lower plots show the side view (x-z).}
  \label{fig:qp}
\end{figure}
\label{sec:results}

\section{Conclusion}
In this paper we have developed an approach for collision avoidance that incorporates robot and obstacle state uncertainties and is solvable in real time. A tight bound for collision probability is provided and a collision constraint for online MPC planning is derived. The tightness of the bound is validated by computing the numerical integral which provides the actual value. A comparison with several state-of-the-art methods is provided and our method is seen to give a tighter upper bound. We also compare the efficiency of our approach with bounding volume methods under varied measurement noise. Multi-robot simulations with both mobile robots and quadrotors are performed to validate our approach. 

Note that currently we assume spherical geometries for robot and obstacles. Such an assumption is valid since any shape can be enclosed inside a sphere or using multiple spheres. However, a more appropriate assumption is to consider the minimum-volume enclosing ellipsoid~\cite{rimon1997JINT}. For every convex polyhedron, there exists a unique ellipsoid of minimal volume that contains the polyhedron and is called the \textit{L\"{o}wner-John ellipsoid} of the polyhedron~\cite{grotschel1988geometric}. Thus the collision constraint is to be formulated based on the distance between the ellipsoids. Furthermore, in some cases the assumption of Gaussian distribution for robot state and other noises may not hold.

\bibliographystyle{splncs03.bst}
\bibliography{/home/antony/Research_Genoa/References/References}

\begin{thebibliography}{10}
\providecommand{\url}[1]{\texttt{#1}}
\providecommand{\urlprefix}{URL }

\bibitem{aoude2013AR}
Aoude, G.S., Luders, B.D., Joseph, J.M., Roy, N., How, J.P.: Probabilistically
  safe motion planning to avoid dynamic obstacles with uncertain motion
  patterns. Autonomous Robots  35(1),  51--76 (2013)

\bibitem{axelrod2018IJRR}
Axelrod, B., Kaelbling, L.P., Lozano-P{\'e}rez, T.: Provably safe robot
  navigation with obstacle uncertainty. The International Journal of Robotics
  Research  37(13-14),  1760--1774 (2018)

\bibitem{bajcsy2019ICRA}
Bajcsy, A., Herbert, S.L., Fridovich-Keil, D., Fisac, J.F., Deglurkar, S.,
  Dragan, A.D., Tomlin, C.J.: A scalable framework for real-time multi-robot,
  multi-human collision avoidance. In: 2019 international conference on
  robotics and automation (ICRA). pp. 936--943. IEEE (2019)

\bibitem{blackmore2011TRO}
Blackmore, L., Ono, M., Williams, B.C.: Chance-constrained optimal path
  planning with obstacles. IEEE Transactions on Robotics  27(6),  1080--1094
  (2011)

\bibitem{bry2011ICRA}
Bry, A., Roy, N.: Rapidly-exploring random belief trees for motion planning
  under uncertainty. In: IEEE International Conference on Robotics and
  Automation. pp. 723--730 (2011)

\bibitem{chow2017JMLR}
Chow, Y., Ghavamzadeh, M., Janson, L., Pavone, M.: Risk-constrained
  reinforcement learning with percentile risk criteria. The Journal of Machine
  Learning Research  18(1),  6070--6120 (2017)

\bibitem{ding2013ICRA}
Ding, X.C., Pinto, A., Surana, A.: Strategic planning under uncertainties via
  constrained markov decision processes. In: IEEE International Conference on
  Robotics and Automation. pp. 4568--4575 (2013)

\bibitem{dutoit2010ICRA}
Du~Toit, N.E., Burdick, J.W.: Robotic motion planning in dynamic, cluttered,
  uncertain environments. In: 2010 IEEE International Conference on Robotics
  and Automation. pp. 966--973. IEEE (2010)

\bibitem{dutoit2011IEEE}
Du~Toit, N.E., Burdick, J.W.: Probabilistic collision checking with chance
  constraints. IEEE Transactions on Robotics  27(4),  809--815 (2011)

\bibitem{falanga2018IROS}
Falanga, D., Foehn, P., Lu, P., Scaramuzza, D.: {PAMPC: Perception-aware model
  predictive control for quadrotors}. In: 2018 IEEE/RSJ International
  Conference on Intelligent Robots and Systems (IROS). pp. 1--8. IEEE (2018)

\bibitem{frey2020RSS}
Frey, K., Steiner, T., How, J.: {Collision Probabilities for Continuous-Time
  Systems Without Sampling}. In: Proceedings of Robotics: Science and Systems.
  Corvalis, Oregon, USA (July 2020)

\bibitem{fridovich2020IJRR}
Fridovich-Keil, D., Bajcsy, A., Fisac, J.F., Herbert, S.L., Wang, S., Dragan,
  A.D., Tomlin, C.J.: Confidence-aware motion prediction for real-time
  collision avoidance1. The International Journal of Robotics Research
  39(2-3),  250--265 (2020)

\bibitem{grotschel1988geometric}
Gr{\"o}tschel, M., Lov{\'a}sz, L., Schrijver, A.: {Geometric Algorithms and
  Combinatorial Optimization}. Springer-Verlag, New York (1988)

\bibitem{hakobyan2019RAL}
Hakobyan, A., Kim, G.C., Yang, I.: {Risk-Aware Motion Planning and Control
  Using CVaR-Constrained Optimization}. IEEE Robotics and Automation Letters
  4(4),  3924--3931 (2019)

\bibitem{hardy2013TRO}
Hardy, J., Campbell, M.: Contingency planning over probabilistic obstacle
  predictions for autonomous road vehicles. IEEE Transactions on Robotics
  29(4),  913--929 (2013)

\bibitem{hewing2018ECC}
Hewing, L., Liniger, A., Zeilinger, M.N.: Cautious nmpc with gaussian process
  dynamics for autonomous miniature race cars. In: 2018 European Control
  Conference (ECC). pp. 1341--1348. IEEE (2018)

\bibitem{Houska2011OCAM}
Houska, B., Ferreau, H., Diehl, M.: {ACADO} {T}oolkit -- {A}n {O}pen {S}ource
  {F}ramework for {A}utomatic {C}ontrol and {D}ynamic {O}ptimization. Optimal
  Control Applications and Methods  32(3),  298--312 (2011)

\bibitem{janson2018ISRR}
Janson, L., Schmerling, E., Pavone, M.: {Monte Carlo motion planning for robot
  trajectory optimization under uncertainty}. In: Robotics Research, pp.
  343--361. Springer (2018)

\bibitem{jasour2019RSS}
Jasour, A.M., Williams, B.C.: Risk contours map for risk bounded motion
  planning under perception uncertainties. Robotics: Science and Systems
  (2019)

\bibitem{johnson1994truncatedGaussian}
Johnson, N.L., Kotz, S., Balakrishnan, N.: Continuous univariate distributions.
  john wiley\& sons. New York, NY  (1994)

\bibitem{kamel2017IROS}
Kamel, M., Alonso-Mora, J., Siegwart, R., Nieto, J.: Robust collision avoidance
  for multiple micro aerial vehicles using nonlinear model predictive control.
  In: 2017 IEEE/RSJ International Conference on Intelligent Robots and Systems
  (IROS). pp. 236--243. IEEE (2017)

\bibitem{lambert2008ICCARV}
Lambert, A., Gruyer, D., Saint~Pierre, G.: {A fast Monte Carlo algorithm for
  collision probability estimation}. In: 10th IEEE International Conference on
  Control, Automation, Robotics and Vision. pp. 406--411 (2008)

\bibitem{lee2013IROS}
Lee, A., Duan, Y., Patil, S., Schulman, J., McCarthy, Z., Van Den~Berg, J.,
  Goldberg, K., Abbeel, P.: Sigma hulls for gaussian belief space planning for
  imprecise articulated robots amid obstacles. In: IEEE/RSJ International
  Conference on Intelligent Robots and Systems. pp. 5660--5667 (2013)

\bibitem{liu2014ICRA}
Liu, W., Ang, M.H.: Incremental sampling-based algorithm for risk-aware
  planning under motion uncertainty. In: 2014 IEEE International Conference on
  Robotics and Automation (ICRA). pp. 2051--2058 (2014)

\bibitem{luo2017PAS}
Luo, Y.z., Yang, Z.: A review of uncertainty propagation in orbital mechanics.
  Progress in Aerospace Sciences  89,  23--39 (2017)

\bibitem{park2012ICAPS}
Park, C., Pan, J., Manocha, D.: {ITOMP: Incremental trajectory optimization for
  real-time replanning in dynamic environments}. In: Twenty-Second
  International Conference on Automated Planning and Scheduling (2012)

\bibitem{park2018IEEE}
Park, C., Park, J.S., Manocha, D.: {Fast and bounded probabilistic collision
  detection for high-DOF trajectory planning in dynamic environments}. IEEE
  Transactions on Automation Science and Engineering  15(3),  980--991 (2018)

\bibitem{patil2012ICRA}
Patil, S., Van Den~Berg, J., Alterovitz, R.: {Estimating probability of
  collision for safe motion planning under Gaussian motion and sensing
  uncertainty}. In: IEEE International Conference on Robotics and Automation.
  pp. 3238--3244 (2012)

\bibitem{rimon1997JINT}
Rimon, E., Boyd, S.P.: Obstacle collision detection using best ellipsoid fit.
  Journal of Intelligent and Robotic Systems  18(2),  105--126 (1997)

\bibitem{sadigh2016RSS}
Sadigh, D., Kapoor, A.: Safe control under uncertainty with probabilistic
  signal temporal logic. Robotics: Science and Systems  (2016)

\bibitem{salzman2017ICAPS}
Salzman, O., Hou, B., Srinivasa, S.: Efficient motion planning for problems
  lacking optimal substructure. In: Twenty-Seventh International Conference on
  Automated Planning and Scheduling (2017)

\bibitem{schmerling2017RSS}
Schmerling, E., Pavone, M.: {Evaluating Trajectory Collision Probability
  through Adaptive Importance Sampling for Safe Motion Planning}. In:
  Proceedings of Robotics: Science and Systems. Cambridge, Massachusetts (July
  2017)

\bibitem{shimanuki2018WAFR}
Shimanuki, L., Axelrod, B.: {Hardness of 3D Motion Planning Under Obstacle
  Uncertainty}. Workshop on Algorithmic Foundations of Robotics  (2018)

\bibitem{sun2016ISRR}
Sun, W., Torres, L.G., Van Den~Berg, J., Alterovitz, R.: Safe motion planning
  for imprecise robotic manipulators by minimizing probability of collision.
  In: Robotics Research, pp. 685--701. Springer (2016)

\bibitem{thomas2020IRIM}
Thomas, A., Mastrogiovanni, F., Baglietto, M.: {Motion Planning with
  Environment Uncertainty}. In: Italian Conference on Robotics and Intelligent
  Machines (I-RIM) (2020)

\bibitem{thomas2021ISR}
Thomas, A., Mastrogiovanni, F., Baglietto, M.: {An Integrated Localization,
  Motion Planning and Obstacle Avoidance Algorithm in Belief Space}.
  Intelligent Service Robotics pp. 1--16 (2021)

\bibitem{zhu2019RAL}
Zhu, H., Alonso-Mora, J.: Chance-constrained collision avoidance for mavs in
  dynamic environments. IEEE Robotics and Automation Letters  4(2),  776--783
  (2019)

\end{thebibliography}
\end{document}